\documentclass{article}

\usepackage{microtype}
\usepackage{graphicx}
\usepackage{subcaption}
\usepackage{booktabs} 

\usepackage{hyperref}

\usepackage{algorithmicx}

\usepackage[accepted]{icml2023}

\usepackage{amsmath}
\usepackage{amssymb}
\usepackage{mathtools}
\usepackage{amsthm}
\usepackage{bbm}
\usepackage{enumitem}
\usepackage[noend]{algpseudocode}

\usepackage[capitalize,noabbrev]{cleveref}

\theoremstyle{plain}
\newtheorem{theorem}{Theorem}[section]

\newtheorem{corollary}[theorem]{Corollary}
\theoremstyle{definition}
\newtheorem{definition}[theorem]{Definition}

\theoremstyle{remark}

\usepackage[textsize=tiny]{todonotes}

\icmltitlerunning{PAC Prediction Sets for Large Language Models of Code}

\begin{document}

\twocolumn[
\icmltitle{PAC Prediction Sets for Large Language Models of Code}




\begin{icmlauthorlist}
\icmlauthor{Adam Khakhar}{a}
\icmlauthor{Stephen Mell}{a}
\icmlauthor{Osbert Bastani}{a}
\end{icmlauthorlist}

\icmlaffiliation{a}{University of Pennsylvania, USA}

\icmlcorrespondingauthor{Adam Khakhar}{ak@alumni.upenn.edu}

\icmlkeywords{Machine Learning, ICML}

\vskip 0.3in
]



\printAffiliationsAndNotice{}  

\begin{abstract}
Prediction sets have recently been shown to be a promising strategy for quantifying the uncertainty of deep neural networks in a way that provides theoretical guarantees. However, existing techniques have largely targeted settings where the space of labels is simple, so prediction sets can be arbitrary subsets of labels. For structured prediction problems where the space of labels is exponential in size, even prediction sets containing a small fraction of all labels can be exponentially large. In the context of code generation, we propose a solution that considers a restricted set of prediction sets that can compactly be represented as partial programs, which are programs with portions replaced with holes. Given a trained code generation model, our algorithm leverages a programming language's abstract syntax tree to generate a set of programs such that the correct program is in the set with high-confidence. Valuable applications of our algorithm include a Codex-style code generator with holes in uncertain parts of the generated code, which provides a partial program with theoretical guarantees. We evaluate our approach on PICARD (a T5 model for SQL semantic parsing) and Codex (a GPT model for over a dozen programming languages, including Python), demonstrating that our approach generates compact PAC prediction sets. This is the first research contribution that generates PAC prediction sets for generative code models.\footnote{Our code is available at \url{https://github.com/adamkhakhar/python-pac-code-prediction-set}.}

\end{abstract}

\section{Introduction}

There has been a great deal of recent interest in uncertainty quantification for deep learning models~\cite{sok-uncertainty-quant}. These techniques are critical for applications of machine learning, where machine learning models are used to guide human decision-makers (e.g., medical or financial decisions), since they convey confidence in the model predictions that can be used to aid downstream decisions.

One promising strategy is \emph{conformal prediction}~\cite{pred-set-iid-1,pred-set-iid-2}, a statistical technique that has recently been adapted to machine learning~\cite{set-cov-shift-1,bastani-pac,pred-set-iid-4}. These techniques focus on constructing \emph{prediction sets}, which capture uncertainty by predicting sets of labels rather than individual labels. A benefit of these approaches is that they come with provable guarantees---typically, that the prediction set includes the ground truth label with high probability.




Most of the work so far has focused on settings where the input $x$ may be complex, but the label $y$ has a relatively simple structure, such as classification and regression.\footnote{Regression problems have an infinite label space, but existing approaches restrict to prediction sets in the form of intervals, which automatically satisfy the monotonicity property described below.} While there is no theoretical obstacle to considering more structured labels, the prediction sets can easily become uninterpretable when the label space is complex as the uncertainty in the model is not easily identifiable. We consider the case of large language models for code generation, where the output is a sequence of tokens. For such models, the output space is exponentially large in the length of the generated sequence, meaning that even if a prediction set contains only a small fraction of labels, it may still be exponentially large.

While recent work has studied structured prediction problems such as object detection and image segmentation~\cite{schulz_behnke}, they avoid this problem by breaking up the prediction space into individual components for which compact prediction sets can be constructed. For instance, for object detection, while the output is a list of bounding boxes, we can construct a prediction set of bounding boxes that may occur, and then separately construct a prediction set for each bounding box.

A natural strategy to construct prediction sets for structured outputs is that rather than considering prediction sets consisting of arbitrary subsets of labels, we can restrict them to ones that can be represented in a compact way. However, existing prediction set algorithms are not designed to search over these structured spaces, meaning that new algorithms are necessary for inferring structured prediction sets.

In this paper, we study prediction sets for code generation, where the labels are sequences of tokens corresponding to valid lines of code. Recent work has demonstrated that large language models based on the GPT architecture~\cite{https://doi.org/10.48550/arxiv.2005.14165} are effective strategies for code generation from context~\cite{https://doi.org/10.48550/arxiv.2107.03374}. For this domain, we propose to represent prediction sets using \emph{partial programs}~\cite{solar2008program},
which are programs with some portions replaced with \emph{holes}. A partial program implicitly represents the prediction set of all programs that can be obtained by filling its holes to produce a valid program. Partial programs are both a natural way of presenting sets of programs to the user and provide needed structure on the space of sets of programs. Prior work has constructed partial programs using large language models~\citep{guo2021learning}, but they do not interpret them as prediction sets, and furthermore do not provide any theoretical guarantees.

To construct PAC prediction sets in this setting, we propose a novel algorithm that modifies an existing one~\cite{bastani-pac} to account for the restricted search space. This algorithm operates by establishing a 1D search space over prediction sets---given a \emph{scoring function} $f(x,y)\in\mathbb{R}$ mapping features $x\in\mathcal{X}$ to labels $y\in\mathcal{Y}$ (typically the probabilities predicted by a traditional neural network).

The main challenge adapting this strategy to prediction sets represented by partial programs is that the search space of partial programs is not 1D. To address this problem, we artificially construct a 1D search space by \emph{pre-selecting} a set of partial programs to consider. As long as this set can be represented in the form in equation (\ref{eqn:predictionset}) for \emph{some} scoring function, then we can use an existing prediction set algorithm to select among prediction sets in this set. The key condition the set must satisfy is \emph{monotonicity} ---i.e., each partial program must be either a superset or subset of each of the others. Then, to compute this set, we devise an integer linear program that encodes the monotonicity constraint along with other constraints on the structure of the partial programs.
We empirically evaluate our approach on 
both PICARD~\cite{picard}, a state-of-the-art semantic parser based on T5~\cite{https://doi.org/10.48550/arxiv.1910.10683},
trained on the Spider dataset~\cite{spider} for SQL semantic parsing, as well as Codex~\cite{https://doi.org/10.48550/arxiv.2107.03374}, a GPT language model fine-tuned on publicly available GitHub code with proficiency in over a dozen programming languages. Our experiments demonstrate that our approach can generate prediction sets of the desired form that satisfy the PAC guarantee, while significantly outperforming a natural baseline in terms of a measure of prediction set size.

\textbf{Example.} In Figure~\ref{fig:target-ast}, we show an example of an SQL query from the Spider dataset along with the \emph{abstract syntax tree} exposing its syntactic structure. In Figure~\ref{fig:predicted-set}, we show the SQL query (and corresponding AST) as predicted by PICARD. Below the predicted query, we show the partial program obtained by our algorithm (it is obtained by deleting the nodes in the AST marked by red crosses). This partial program represents the prediction set of all programs that can be obtained by filling the ?? portions with some expressions. It is guaranteed to contain the ground truth query with high probability.
 
\textbf{Contributions.} Our contributions include the notion of partial programs as prediction sets for code generation, an algorithm for constructing PAC prediction sets in this setting, and an empirical evaluation demonstrating the efficacy of our algorithm. Finally, while we focus on code generation, we believe our techniques can be adapted to construct prediction sets for other structured prediction problems.

\begin{figure*}[h]
\begin{subfigure}{.45\textwidth}
\centering
\includegraphics[width=3.1 in]{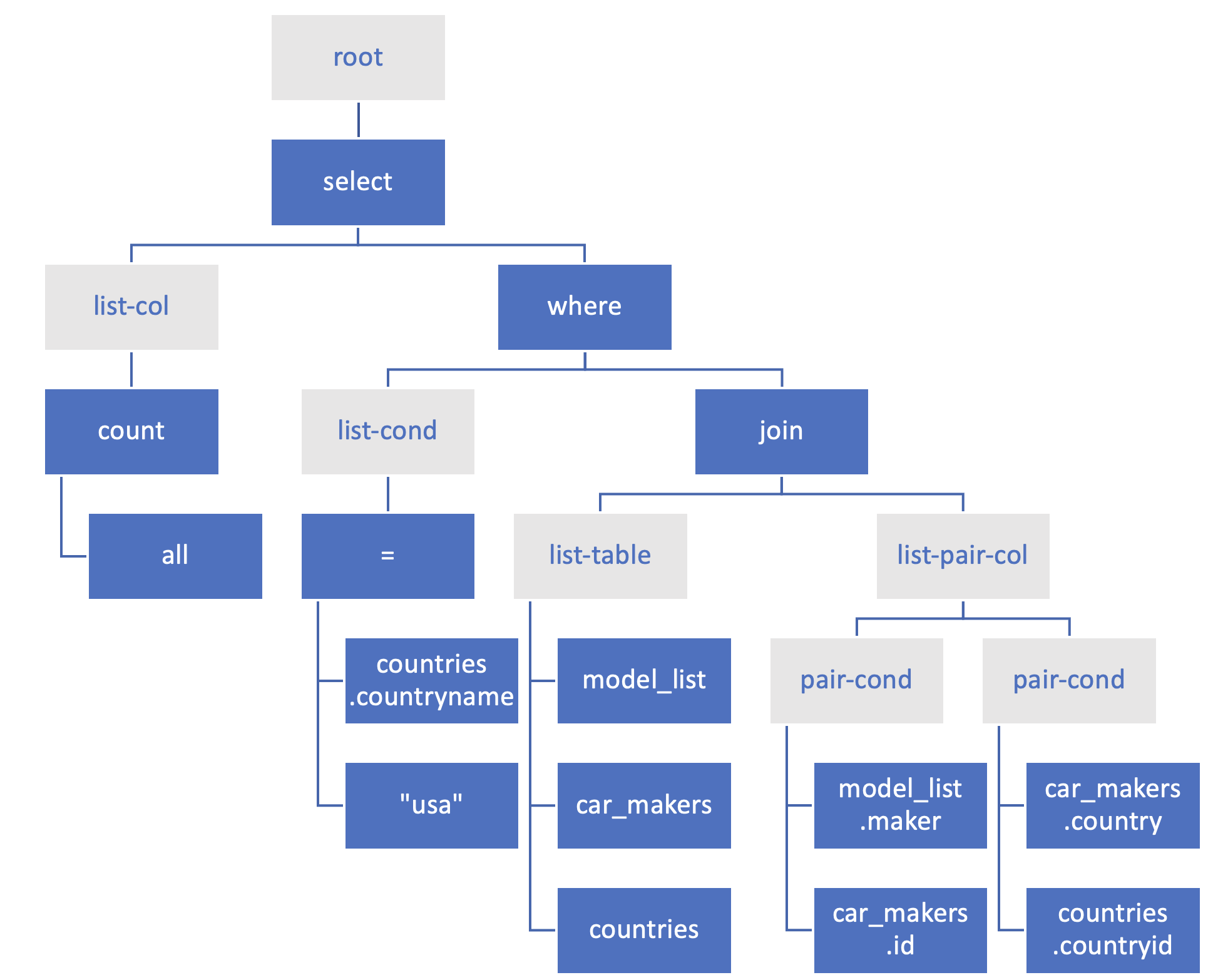} \\
{\scriptsize
\begin{align*}
&\texttt{SELECT COUNT(*) FROM countries AS t1} \\
&\texttt{JOIN car\_makers AS t2 on t1.countryid = t2.country} \\
&\texttt{JOIN model\_list as t3 on t2.id=t3.maker} \\
&\texttt{WHERE t1.countryname = "usa";}
\end{align*}
}
\caption{Ground truth abstract syntax tree (AST) and SQL query from the Spider dataset~\cite{spider}.}
\label{fig:target-ast}
\end{subfigure}
\hfill
\begin{subfigure}{.45\textwidth}
\centering
\includegraphics[width=3.1 in]{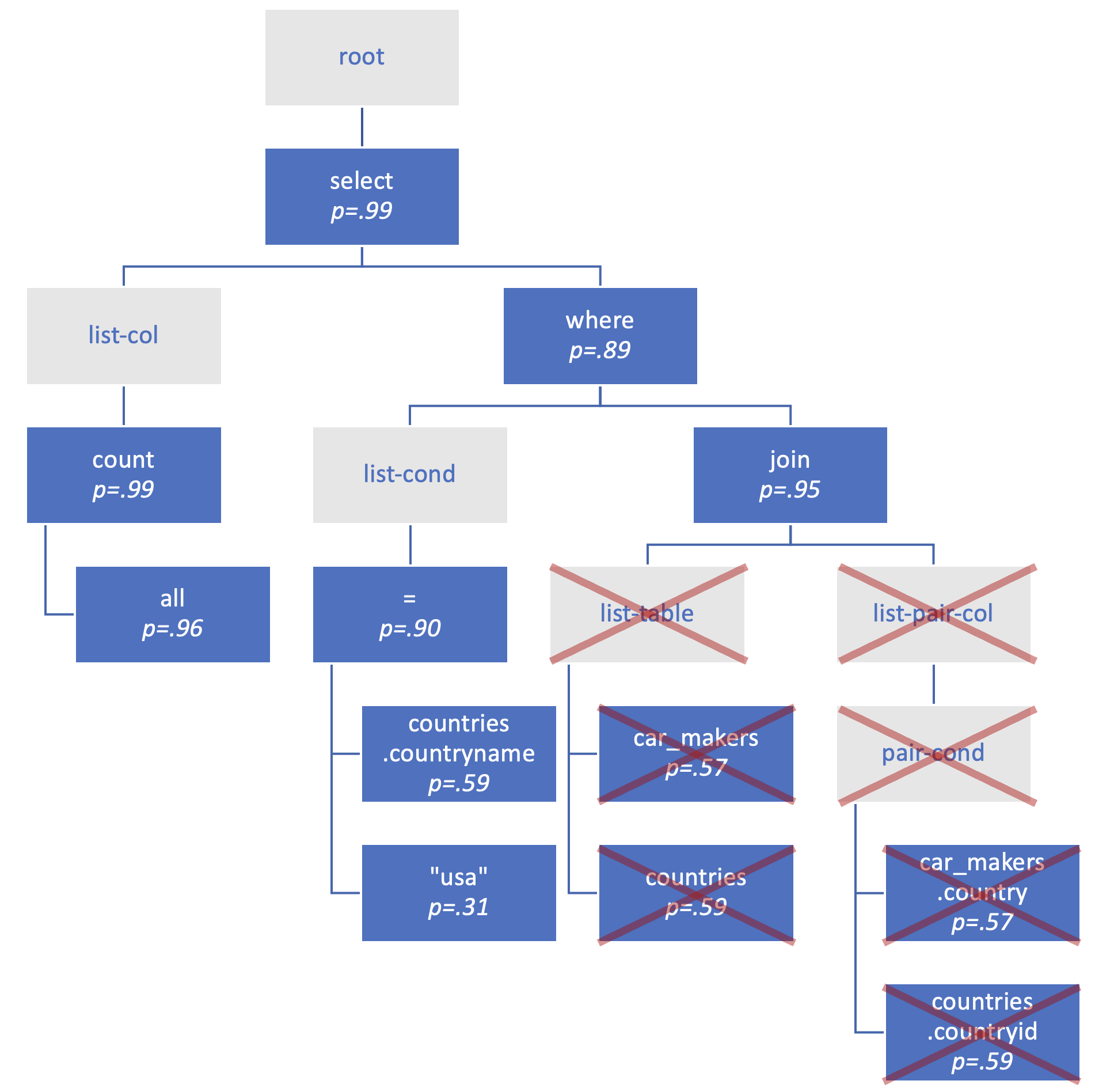}
{\scriptsize
\begin{align*}
&\texttt{SELECT COUNT(*) FROM countries AS t1} \\
&\texttt{JOIN car\_makers as t2 on t1.countryid = t2.country} \\
&\texttt{WHERE t1.countryname = "usa";} \\\\
&\texttt{SELECT COUNT(*) FROM countries AS t1} \\
&\texttt{JOIN ?? on ?? WHERE t1.countryname = "usa";}
\end{align*}}
\caption{Predicted AST, predicted SQL query, and prediction set for the same task as in Figure~\ref{fig:target-ast} with $m=2$ holes.}
\label{fig:predicted-set}
\end{subfigure}
\caption{
Example of ground truth SQL query, predicted SQL query, and the constructed prediction set. Note that the predicted query is incorrect. With the two holes in the SQL query, the predicted AST code prediction set contains the ground truth query.}
\label{fig:example}
\end{figure*}

\section{Background}

In this section, we introduce the notion of PAC prediction sets, as well as the code generation task.

\subsection{PAC Prediction Sets} \label{PAC-SET}

PAC prediction sets based on conformal prediction have recently been proposed as a rigorous way to quantify uncertainty for deep neural networks~\cite{bastani-pac}. A \emph{prediction set model} is a function $F:\mathcal{X}\to2^{\mathcal{Y}}$ mapping inputs $x$ to sets of labels $F(x)\subseteq\mathcal{Y}$. We typically construct prediction sets based on a given \emph{scoring function} $f:\mathcal{X}\times\mathcal{Y}\to\mathbb{R}$, which maps features $x\in\mathcal{X}$ to labels $y\in\mathcal{Y}$ (with higher scores corresponding to higher likelihood of being the true label). Scoring functions are often neural networks. Given $f$, we consider the family of prediction set models with a single parameter $\tau \in \mathbb{R}$: 
\begin{equation} \label{eqn:predictionset}
F_{\tau}(x) \coloneqq \{y \in \mathcal{Y} \mid f(x,y) \geq \tau \},
\end{equation}
i.e., the set of all labels with score at least $\tau$. To define correctness, we assume that the data is sampled i.i.d. according to some distribution $p(x,y)$.
\begin{definition}
A parameter $\tau \in \mathbb{R}_{\geq 0}$ is \emph{$\epsilon$-correct} if
\begin{align*}
\mathbb{P}_{p(x,y)}\left[y \in F_{\tau}(x)\right] \geq 1 - \epsilon.
\end{align*}
\end{definition}
We let $\mathcal{T}_{\epsilon}\subseteq\mathbb{R}$ denote the set of $\epsilon$-correct $\tau$. Next, consider an estimator $\hat\tau:\mathcal{Z}^*\to\mathbb{R}$, where $\mathcal{Z}=\mathcal{X}\times\mathcal{Y}$ (with $\mathcal{Z}^*=\bigcup_{i=1}^\infty\mathcal{Z}^i$) is the set of calibration datasets. We let $p(Z)=\prod_{(x,y)\in Z}p(x,y)$ be the distribution over datasets.
\begin{definition}
\label{def:pacpredictionset}
An estimator $\hat{\tau}$ is \emph{$(\epsilon, \delta)$-probably approximately correct (PAC)} if
\begin{align*}
\mathbb{P}_{p(Z)}\left[\hat{\tau}(Z) \in \mathcal{T}_{\epsilon}\right] \geq 1 - \delta.
\end{align*}
\end{definition}
Our goal is to construct PAC prediction sets that are small in size, where size is defined as
\begin{align*}
\bar{S}(\tau)=\mathbb{E}_{p(x,y)}[S(F_\tau(x))],
\end{align*}
for some given size metric $S:2^{\mathcal{Y}}\to\mathbb{R}_{\ge0}$. Assuming $S$ is monotone (i.e., $Y\subseteq Y'$ implies $S(Y)\le S(Y')$), then larger values of $\tau$ correspond to smaller sizes. Thus, our goal is to maximize $\tau$ while guaranteeing that $\tau\in \mathcal{T}_{\epsilon}$.

For this setting, there exists an estimator $\hat{\tau}$ to construct a PAC prediction set. These algorithms optimize $\tau$ subject to a constraint on the number of errors in the validation set:
\begin{align}
\label{eqn:pacalgo}
\hat{\tau}(Z)=\operatorname*{\arg\max}_{\tau\in\mathbb{R}}\tau
~~\text{subj. to}~~
\sum_{i=1}^n\mathbbm{1}(y_i\not\in F_\tau(x_i))\le k,
\end{align}
where $k$ can be computed from $\epsilon$, $\delta$, and $n$ (see~\cite{bastani-pac} for details). In other words, we choose the largest $\tau$ subject to a bound $k$ on the number of errors made by the prediction set.

\subsection{Code Generation}

We consider the problem of generating code in some language, where the label is a sequence of tokens---i.e., we have $\mathcal{Y}=\Sigma^*$, where $\Sigma$ is the set of tokens. While our approach is general, we focus on strategies based on large language models (LLMs), where the input $x$ consists of the generation context (e.g., several lines of code preceding the line to be generated); recent work has demonstrated that LLMs are very effective models for solving this problem~\cite{textgen}. In this strategy, tokens represent parts or whole words~\cite{bpe}, which are then predicted, either autoregressively~\cite{https://doi.org/10.48550/arxiv.2210.14698} or sequence-to-sequence~\cite{https://doi.org/10.48550/arxiv.1409.3215}, to form a full target when concatenated.

\subsection{Abstract Syntax Trees}

Our strategy for constructing prediction sets relies on the \emph{abstract syntax tree} of the generated code. This data structure is similar to a parse tree obtained by using a context-free grammar (CFG) parser to parse a sentence, but aims to represent constructs in the code (e.g., variables, assignment statements, if statements, etc.) while abstracting away low-level syntax (e.g., parentheses matching).

Formally, given a (valid) sequence of tokens $y\in\Sigma^*$, an AST is a tree $T=(V,E)$, with nodes $v\in V$ and edges $(v,v')\in E\subseteq V\times V$. Each node $v\in V$ represents some subsequence of tokens in $y$ (the subsequence of a node must be contained in the subsequence of its parent). This subsequence is given by a mapping $\eta:V\to\mathbb{N}^2$, where $\eta(v)=(i,j)$ indicates that node $v$ corresponds to subsequence $y_i...y_{j-1}$ of $y$. Then, we assume that if $\eta(v)=(i,j)$, then its parent $v'$
satisfies $\eta(v')=(i',j')$ for some $i'\le i<j\le j'$ (here, the parent $v'$ is the unique node $v'\in V$ such that there is an edge $(v',v)\in E$).

ASTs are language-specific, and can also vary from one implementation to another; we assume the AST construction is provided by the user; standard implementions exist for all programming languages since ASTs are required to compile or interpret programs. Figure~\ref{fig:target-ast} depicts an examples of an AST for an SQL query. In our approach, we use the AST to restrict the space of prediction sets we consider. Intuitively, we consider prediction sets that can be represented by replacing some subtree of the AST with a ``hole'' representing an arbitrary sequence of tokens.

\section{PAC Prediction Sets for Code Generation}
\label{sec:problem}

In this section, we formalize the notion of a PAC prediction set for code generation.

\subsection{Partial Programs}

Our algorithm takes as input a trained model $f$ represented as a scoring function $f:\mathcal{X}\times\mathcal{Y}\to\mathbb{R}$, along with a calibration dataset $D$ of input-output pairs $Z = \{(x_i,y_i)\}_{i=1}^n$. We assume an algorithm that, given an input $x$, uses $f$ to produce a label $y=\bar{f}(x)$---e.g., we may compute $y$ using a greedy decoding strategy.

Our goal is to use $f$ to construct a PAC prediction set. The challenge is that there are an enormous number of possible labels (exponential in the length of the generated sequence), meaning a traditional prediction set will not be human-interpretable. Thus, we need to constrain the search space over prediction set models.

To this end, we consider prediction sets defined by \emph{partial programs}, which are sequences with special tokens called \emph{holes}; each hole can be \emph{filled} with an arbitrary subsequence, and filling all holes produces a complete sequence. Formally, we denote the special token by $??$; then, a partial program is a sequence $y\in(\Sigma\cup\{??\})^*$. Partial programs also require that holes occur in a way compatible with the context-free grammar defining the programming language; we describe this condition in Section~\ref{sec:aststructure}. Suppose that $y$ has $k$ holes:
\begin{align*}
y=y_0\;??\;y_1...\;y_{k-1}??\;y_k,
\end{align*}
where $y_0,y_1,...,y_k\in\Sigma^*$. Then, given $\bar{y}_1,...,\bar{y}_k\in\Sigma^*$, we can fill the holes in $y$ using these sequences to obtain
\begin{align*}
\bar{y}
=\text{fill}(y;\bar{y}_1,...,\bar{y}_k)
=y_0\bar{y}_1y_1...y_{k-1}\bar{y}_ky_k,
\end{align*}
i.e., replace the $i$th hole with $\bar{y}_i$. We call $\bar{y}\in\Sigma^*$ a \emph{completion} of $y$; we denote the set of completions of $y$ by
\begin{align*}
\pi(y)\coloneqq\{\bar{y}\in\Sigma^*\mid\exists y_1,...,y_k\in\Sigma^*\,.\,\bar{y}=\text{fill}(y;y_1,...,y_k)\},
\end{align*}
i.e., all possible ways in which the holes in $y$ can be filled (note that $\pi(y)$ may be an infinite set).

Now, our strategy is to restrict prediction sets to sets of programs represented by partial programs. At a high level, our algorithm starts with a concrete program $\bar{y}=\bar{f}(x)$, and then replaces a certain number of subsequences in $\bar{y}$ with holes to obtain a partial program $y$ such that $\pi(y)$ contains the true program $y^*$ with high probability.

\subsection{Leveraging AST Structure}
\label{sec:aststructure}

An additional constraint is that we want to remove subsequences that correspond to full ``units of code''. For example, if we start with the program \texttt{print([1,2,3])}, we might remove a subsequence to obtain \texttt{print([??)}; however, this partial program can be hard for the user to understand since the brackets do not match. Instead, we may want to require the algorithm to either remove only the elements of the array to obtain \texttt{print([??])}, or remove the full array to obtain \texttt{print(??)}. Thus, we only consider removing subsequences corresponding to nodes in the AST $T=(V,E)$ of $\bar{y}$---i.e., sequences $\bar{y}_i...\bar{y}_{j-1}$ with a hole $??$, where $(i,j)=\eta(v)$ for some $v\in V$.

\subsection{Bounded Number of Holes}

Even with the AST constraint, we can still obtain very complicated partial programs by removing a large number of leaf nodes. For example, we could remove three nodes from \texttt{print([1,2,3])} to obtain the partial program \texttt{print([??,??,??])}; for longer lists, the resulting partial program would be even more complex. A simpler solution would be to leave the entire contents of the list as a single hole: \texttt{print([??])}. To this end, we impose a constraint that at most $m$ subtrees of the AST are replaced with holes, resulting in a partial program with at most $m$ holes. Here, $m\in\mathbb{N}$ is a given hyperparameter; for instance, we may take $m$ to be 1-3 holes. In our experiments, we find that $m=1$ works well in practice.


\subsection{Problem Formulation}

Our goal is to design an algorithm $\mathcal{A}$ that maps a validation set $Z=\{(x_i,y_i)\}_{i=1}^n$ and a new input $x\in\mathcal{X}$ to a partial program $y=\mathcal{A}(x;Z)$ such that $\mathcal{A}$ satisfies the PAC property given in Definition~\ref{def:pacpredictionset}---i.e., we have
\begin{align*}
\mathbb{P}_{p(Z)}\left[\mathbb{P}_{p(x,y)}\left[y\in\pi(\mathcal{A}(x;Z))\right]\ge1-\epsilon\right]\ge1-\delta,
\end{align*}
where $p(Z)=\prod_{i=1}^np(x_i,y_i)$. As before, we additionally want to minimize the expected size of $\pi(\mathcal{A}(x;Z))$. As we describe in the next section, our algorithm constructs $y=\mathcal{A}(x;Z)$ by starting from the highest probability prediction $\bar{y}=\operatorname*{\arg\max}f(x,y)$ (more specifically, an approximate maximum based on greedy decoding), and then removes subtrees of the AST of $\bar{y}$ to obtain $y$. Then, we define size to be the fraction of nodes retained in $y$---i.e., $S(y)=|T'|/|T|$, where $T'$ is the AST of $y$ and $T$ is the AST of $\bar{y}$.

\section{Structured Prediction Set Algorithm} \label{algo}

We describe our algorithm $\mathcal{A}$ (in Algorithm~\ref{alg:main}) for constructing structured PAC prediction sets for code generation.

\subsection{General Strategy}

Recall that existing approaches to constructing prediction sets rely on a 1D search space over parameters $\tau\in\mathbb{R}$. Our approach is to design such a 1D search space and then apply existing prediction set algorithms.  However, we cannot directly use the scoring function $f(x,y')$ to construct prediction sets, since its level sets $\{y'\mid f(x,y')\ge\tau\}$ may not have the desired structure described in section~\ref{sec:problem}---i.e., they may have not have form $\pi(y)$ for some partial program $y$.

Instead, we design a modified scoring function $\tilde{f}(x,y')$
whose level sets have the form $\{y'\mid\tilde{f}(x,y')\ge\tau\}=\pi(y)$ for some partial program $y$.
To achieve this goal, we first note that for a single input $x$, if our goal is to obtain a prediction set of the form $\pi(y)$ 
for some partial program $y$, we need $\tilde{f}(x,y')>\tau$ for all $y'\in\pi(y)$ and $\tilde{f}(x,y')<\tau$ otherwise. For instance, we could define $\tilde{f}(x,y')=\tau+\gamma$ for all $y'\in\pi(y)$ (for an arbitrarily small $\gamma\in\mathbb{R}_{>0}$), and $\tilde{f}(x,y')=\tau-\gamma$ otherwise.

More generally, our strategy can be envisioned as follows:
\begin{enumerate}[topsep=0pt,itemsep=0ex,partopsep=1ex,parsep=0ex,leftmargin=3ex]
\item Design an algorithm $\tilde{\mathcal{A}}$ such that for every input $x$ and parameter $\tau$, it outputs a partial program $y=\tilde{\mathcal{A}}(x,\tau)$ that corresponds to a desired prediction set $\pi(y)$.
\item Construct a modified scoring function $\tilde{f}$ such that $\{y'\mid\tilde{f}(x,y')\ge\tau\}=\pi(y)$ for some $y$.
\item Use an existing PAC prediction set algorithm to choose a threshold $\hat{\tau}(Z)$ based on $\tilde{f}$.
\end{enumerate}
The key constraint on $\tilde{\mathcal{A}}$ is that we have to be able to construct $\tilde{f}$. We saw that we can construct $\tilde{f}$ for a single triple $(x,\tau,y)$, but given multiple triples, such a $\tilde{f}$ may not exist.

For $\tilde{f}$ to exist, these triples must satisfy \emph{monotonicity}---i.e., for two parameters $\tau,\tau'\in\mathbb{R}$ such that $\tau\le\tau'$, the corresponding prediction sets satisfy $F_{\tau}(x)\supseteq F_{\tau'}(x)$. In other words, as the threshold on the score decreases, the size of the prediction set becomes larger \emph{uniformly} for all inputs $x\in\mathcal{X}$. Thus, we need to ensure that our partial programs also satisfy this property---i.e., if $\tau\le\tau'$, then letting $y=\tilde{\mathcal{A}}(x,\tau)$ and $y'=\tilde{\mathcal{A}}(x',\tau)$, we have $\pi(y)\supseteq\pi(y')$.

We impose a stronger constraint that implies monotonicity: if $\tau\le\tau'$, then we require that every node in $y=\tilde{\mathcal{A}}(x,\tau)$ also occurs in $y'=\tilde{\mathcal{A}}(x,\tau')$. It is easy to see that if the AST of $y$ contains a subset of the nodes in the AST of $y'$, then $\pi(y)\supseteq\pi(y')$. For simplicity, we refer to this stronger constraint as monotonicity for the remainder of the paper.

\begin{theorem}
\label{thm:main}
Assume that $\tilde{\mathcal{A}}$ is monotone. There exists a scoring function $\tilde{f}$ such that for all $x\in\mathcal{X}$, we have
\begin{align*}
\pi(\tilde{\mathcal{A}}(x,\tau))=\tilde{F}_{\tau}(x)\coloneqq\{y'\mid\tilde{f}(x,y')\ge\tau\}
\end{align*}
for all $\tau\in\mathbb{R}$ except a finite subset.
\end{theorem}
\begin{proof}
Recall that we have assumed that we choose $y$ to be a subset of the most likely program $\bar{y}=\operatorname*{\arg\max}_{y'}f(x,y)$ (i.e., remove some subset of nodes in $\bar{y}$). Thus, the space of possible partial programs $y$ for a given input $x$ is finite. Suppose the partial programs encountered by enumerating $\tau$ from $+\infty$ to $-\infty$ is $y_1,...,y_k$; this chain must be finite due to monotonicity. Also, $\pi(y_1)\supseteq\pi(y_2)\supseteq...\supseteq\pi(y_k)$.

Next, let the value of $\tau$ at which $\tilde{\mathcal{A}}(x,\tau)$ changes from $y_{i-1}$ to $y_i$ be $\tau_i$, and define $\tau_1=-\infty$ and $\tau_{k+1}=+\infty$. Then, we have $\tilde{\mathcal{A}}(x,\tau)=y_i$ for $\tau_i<\tau<\tau_{i+1}$ (the value at $\tau_i$ can be either $y_i$ or $y_{i-1}$). Now, we can define
\begin{align*}
\tilde{f}(x,y')=\tau_i\quad\text{where}\quad
y'\in\pi(y_i)\wedge y'\not\in\pi(y_{i+1}),
\end{align*}
noting that by monotonicity, the condition holds for at most one $i$; if it does not hold for any $i$, but $y'\in\pi(y_k)$, then we take $\tilde{f}(x,y')=\infty$; otherwise, we take $\tilde{f}(x,y')=-\infty$.\footnote{If we assume that $y_1=\;??$ is the partial program consisting of a single hole, $y'\in\pi(y_1)$ for all $y'$, so this last case is unnecessary.}
This strategy ensures that the scores $\tilde{f}(x,y')$ fall between the $\tau_i$. It is easy to see that with this choice of $\tilde{f}$, we have
\begin{align*}
\pi(y_i)=\{y'\mid\tilde{f}(x,y')\ge\tau\}
\end{align*}
for all $\tau\in\mathbb{R}\setminus\{\tau_1,...,\tau_k\}$.
By construction, we have $\tilde{\mathcal{A}}(x,\tau)\in\{y_i\}_{i=1}^k$, so the claim follows.
\end{proof}

In our algorithm, can avoid problematic values of $\tau$ (i.e., the finite subset excluded in the statement of Theorem~\ref{thm:main}) simply by reducing $\hat\tau(Z)$ by a tiny amount---i.e., we use $\hat\tau(Z)-\gamma$ for an arbitrarily small $\gamma\in\mathbb{R}_{>0}$.
\begin{corollary}
\label{cor:main}
For sufficiently small $\gamma\in\mathbb{R}_{>0}$, the prediction set $\pi(\tilde{\mathcal{A}}(x,\hat\tau(Z)-\gamma))$ is PAC, where $\hat\tau(Z)$ is obtained by using an existing PAC prediction set algorithm with $\tilde{f}$.
\end{corollary}
\begin{proof}
Note that for sufficiently small $\gamma$, either $\hat\tau(Z)$ or $\hat\tau(Z)-\gamma$ is not problematic. If the former holds, then
\begin{align*}
\pi(\tilde{\mathcal{A}}(x,\hat\tau(Z)-\gamma))
\supseteq
\pi(\tilde{\mathcal{A}}(x,\hat\tau(Z)))
=
\tilde{F}_{\hat\tau(Z)}(x).
\end{align*}
If the latter holds, then
\begin{align*}
\pi(\tilde{\mathcal{A}}(x,\hat\tau(Z)-\gamma))
=
\tilde{F}_{\hat\tau(Z)-\gamma}(x)
\supseteq
\tilde{F}_{\hat\tau(Z)}(x).
\end{align*}
By assumption, $\tilde{F}_{\hat\tau(Z)}$ is PAC, so the claim follows.
\end{proof}
As a consequence, we can use $\pi(\tilde{\mathcal{A}}(x,\hat\tau(Z)-\gamma))$ as our PAC prediction set. It remains to describe our design of $\tilde{\mathcal{A}}$.

\subsection{Probabilities for ASTs}

First, we describe the objective that our choice of $\tilde{\mathcal{A}}$ uses to prioritize partial programs. A standard strategy for ranking candidate prediction sets is to consider the probability mass of labels in the set~\cite{angelopoulos2020uncertainty}. Then, we can prioritize prediction sets that cover higher probability mass, since they are more likely to contain the true label.

In our setting, the corresponding principle is to prioritize replacing portions of the program that are more uncertain (according to the original scoring function $f$) with holes, since these portions are the most likely to be incorrect compared to the true program. In particular, since the corresponding prediction set is constructed by filling holes with all possible subsequences, and since low-probability subtrees are the ones where other subsequences have higher probability, so replacing these subtrees with holes most increases the aggregate probability mass of the prediction set.

To formalize this notion, we assume that the AST of the top prediction $\bar{y}=\operatorname*{\arg\max}_y f(x,y)$ has probabilities associated with its leaf nodes.\footnote{Here, we assume access to the AST $T$, which is the case for all major languages (since the ability to construct $T$ is needed to interpret or compile the program $\bar{y}$). In addition, we assume that $\bar{y}$ is valid; in our experiments, we found that invalid code or unparseable output appears in less than 0.1\% of cases. When an invalid AST is produced, we can simply take additional samples; alternatively, techniques like PICARD impose constraints during generation to ensure that only valid programs are decoded.}
In particular, we assume that each leaf node $v$ in the AST $T$ of $\bar{y}$ is labeled with a value $\ell_v\in\mathbb{R}_{\ge0}$ denoting the negative log probability of $v$ conditioned on its ancestors\footnote{This probability model assumes that nodes are generated conditioned only on their ancestors. In practice, we often use sequence models that do not respect the structure of $T$, but as a heuristic, we can still use the predicted probability of the token labeling each leaf $v$ to construct $\ell_v$. Because the scoring function can be arbitrary, this heuristic does not invalidate our coverage guarantee.}---i.e., $\ell_v=-\log p(v\mid v_1,...,v_k)$.
Then, the negative log probability of the entire AST $T$ is
\begin{align*}
\ell_T=\sum_{v\in T}\ell_v.
\end{align*}
Now, if we replace a subtree with holes, we are considering enumerating all possible subtrees to fill that hole construct the prediction set. Thus, the aggregate negative log probability of that subtree goes from 
$\sum_{v\in\text{subtree}}\ell_v$ to $0$ (i.e., its probability becomes $1$).
That is, if we replace a subtree with a hole, we can simply drop those nodes from the sum in $\ell_T$.

Thus, we can label holes $v$ in an AST $T'$ with negative log probability $\ell_v=1$. Then, we can extend the definition for the negative log probability of a tree to trees with holes:
\begin{align*}
\ell_{T'}=\sum_{v\in T'}\ell_v.
\end{align*}
We use this objective to prioritize partial programs.

\begin{figure*}[h]
\begin{subfigure}{0.45\textwidth}
\includegraphics[width=\textwidth]{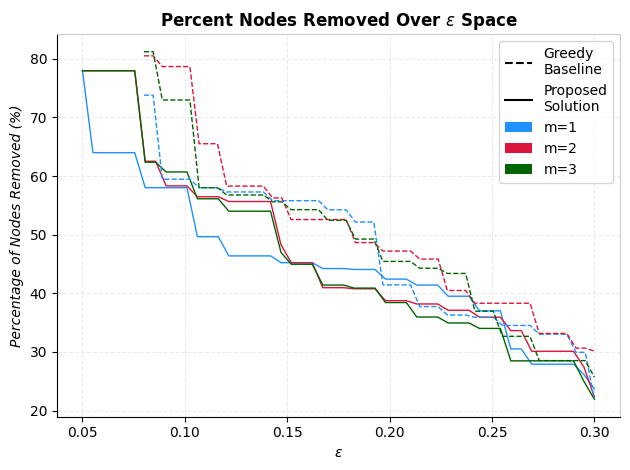}
\caption{Fraction of nodes removed as a function of $\epsilon$ for varying bound $m$ on the number of holes.}
\label{fig:spider_node_rm}
\end{subfigure}
\hfill
\begin{subfigure}{0.45\textwidth}
\includegraphics[width=\textwidth]{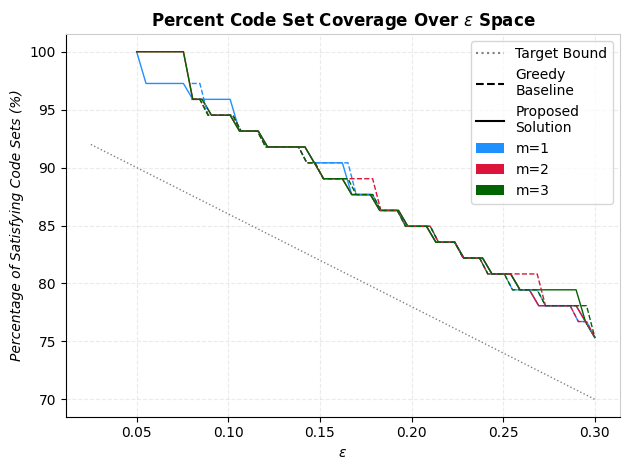}
\caption{Prediction set coverage rate as a function of $\epsilon$. The coverage always exceeds the desired $1-\epsilon$ rate (dotted line).}
\label{fig:spider_set_coverage}
\end{subfigure}
\caption{Node removal and code set coverage for varying number of holes and varying $\epsilon$ values. The experiment is conducted using the PICARD~\cite{picard} model evaluated on the SQL dataset~\cite{spider}.}
\label{fig:sql-experiments}
\end{figure*}

\begin{figure*}[h]
\begin{subfigure}{0.45\textwidth}
\centering
\includegraphics[width=\textwidth]{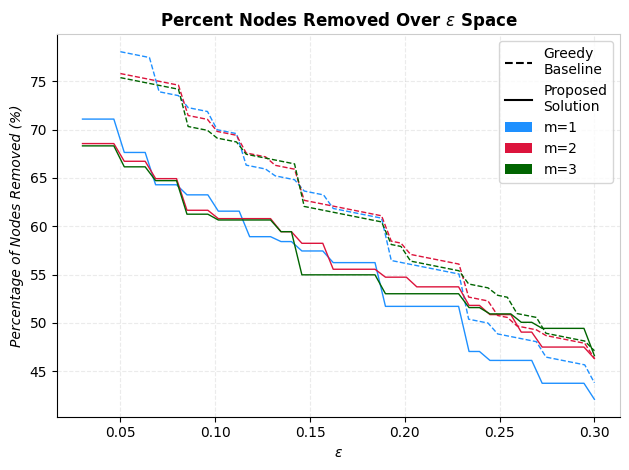}
\caption{Fraction of nodes removed as a function of $\epsilon$ for varying bound $m$ on the number of holes.}
\label{fig:codex_node_rm}
\end{subfigure}
\hfill
\begin{subfigure}{0.45\textwidth}
\centering
\includegraphics[width=\textwidth]{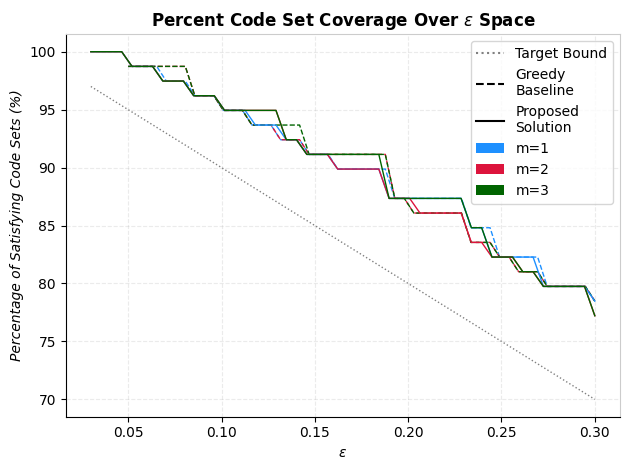}
\caption{Prediction set coverage as a function of $\epsilon$. The coverage always exceeds the desired $1-\epsilon$ rate (dotted line).}
\label{fig:codex_set_coverage}
\end{subfigure}
\caption{Node removal and code set coverage for varying number of holes and varying $\epsilon$ values. The experiment is conducted using OpenAI's Codex Model~\cite{https://doi.org/10.48550/arxiv.2107.03374} evaluated on the Python datasets \cite{hendrycksapps2021, https://doi.org/10.48550/arxiv.2107.03374}.}
\label{fig:codex-experiments}
\end{figure*}

\subsection{Computing Partial Programs via Optimization}

Next, we describe our algorithm $\tilde{\mathcal{A}}$ for constructing 
a partial program $y=\tilde{\mathcal{A}}(x,\tau)$ representing a prediction set for a given input $x$ and parameter $\tau$. Rather than compute this output for all possible $\tau$, we fix a finite subset of parameters $\tau_1<\tau_2<...<\tau_k$, and construct partial programs $y_1,y_2,...,y_k$ for these values. Then, we take $\tilde{\mathcal{A}}(x,\tau)=y_i$, where $\tau\in(\tau_{i-1},\tau_i]$. Then, $\tilde{\mathcal{A}}$ is formulated as an integer linear program (ILP), which we describe below.

\textbf{Optimization variables.}
Our program includes two sets of variables. First, for each node $v$ in $T$, where $T$ is the AST for $\bar{y}$, and for each parameter $\tau_i$ we include a binary variable $\alpha_{i,v}\in\{0,1\}$ indicating whether we remove the subtree rooted at $v$ from $\bar{y}$ to construct $y_i$. Second, for each node $v$, we include a binary variable $\beta_{i,v}\in\{0,1\}$ indicating whether we remove $v$ from $\bar{y}$ to construct $y_i$. We separately track removing subtrees from removing nodes so we can impose our bound on the number of subtrees removed.

\begin{algorithm}[t]
\caption{Our structured PAC prediction set algorithm.}
\label{alg:main}
\begin{algorithmic}
\Procedure{PredictionSet}{Scoring function $f$, Validation dataset $Z$, Confidence levels $\epsilon,\delta$}
\For{$(x,\bar{y})\in Z$}
\State $\alpha,\beta\gets$ Solve Eqs.~\ref{eqn:opt1}-\ref{eqn:opt7} for $(x,y)$
\State $y_i\gets$ Remove nodes $v$ such that $\alpha_{i,v}=1$
\EndFor
\State \textbf{return} $\tau_{i^*}$, where $i^*$ is the largest $i$ such that
\begin{align*}
\sum_{(x,\bar{y})\in Z}\mathbbm{1}(\bar{y}\not\in\pi(y_i))\le k(\epsilon,\delta,|Z|)
\end{align*}
\EndProcedure
\end{algorithmic}
\end{algorithm}

\textbf{Overall program.} Overall, our goal is to minimize the leaf nodes removed on average across $y_i$:
\begin{align}
\min_{\alpha,\beta}~\sum_{i=1}^k\sum_{v\in T}\beta_{i,v} \label{eqn:opt1}
\end{align}
subject to the following constraints:
\begin{align}
&\sum_{v\in V}\alpha_{i,v}\le m
\qquad(\forall i\in[k]) \label{eqn:opt2} \\
&\alpha_{i,v}\to\beta_{i,v}
\qquad(\forall v\in V,~i\in[k]) \label{eqn:opt3} \\
&\beta_{i,v}\to\beta_{i,v'}
\qquad(\forall(v,v')\in E) \label{eqn:opt4} \\
&\beta_{i,v}\rightarrow\alpha_{i,v}\vee\beta_{i,v'}
\qquad(\text{where }(v',v)\in E) \label{eqn:opt5} \\
&\beta_{i,v}\rightarrow\beta_{i+1,v}
\qquad(v\in V,~\forall i\in\{2,...,k\}) \label{eqn:opt6} \\
&\sum_{v\in T}\ell_v\cdot(1-\beta_{i,v})\le\tau_i
\qquad(\forall i\in[k]) \label{eqn:opt7}
\end{align}
We have used Boolean notation for clarity; we can enforce $\alpha\to\beta$ for $\alpha,\beta\in\{0,1\}$ via the constraint $\alpha\le\beta$. We describe these constraints in detail below. Once we have solved this program, our algorithm directly constructs $y_i$ by removing all subtrees for which $\alpha_{i,v}=1$.

\textbf{Removal constraints.}
We include constraints bounding how many subtrees we remove, and enforcing that if we remove a subtree, we remove all nodes in that subtree:
\begin{itemize}[topsep=0pt,itemsep=0ex,partopsep=1ex,parsep=0ex,leftmargin=3ex]
\item Eq.~\ref{eqn:opt2}: We remove at most $m$ subtrees.
\item Eq.~\ref{eqn:opt3}: If we remove the subtree at $v$, then we remove $v$.
\item Eq.~\ref{eqn:opt4}: If we remove $v$ and $v'$ is a child of $v$, then we also remove $v'$.
\item Eq.~\ref{eqn:opt5}: If we remove $v$, then we either remove the subtree at $v$ or we remove its parent $v'$.
\item Eq.~\ref{eqn:opt6}: If we remove $v$ for $y_i$, then we also remove it for $y_{i+1}$ (i.e., enforce monotonicity).
\end{itemize}

\textbf{Probability mass constraint.}
Finally, Eq.~\ref{eqn:opt7} ensures that we remove sufficient probability mass from $T$ so $\pi(y_i)$ meets the $\tau_i$ threshold---in particular, it requires that the negative log likelihood of the AST $T_i$ of $y_i$ is below $\tau_i$. Note that we use $1-\beta_{i,v}$ since we are summing over nodes remaining in $T_i$. Also, note that $\ell_v$ is a real-valued constant.

\subsection{Structured PAC prediction sets.}

Given the partial programs $y_i=\tilde{\mathcal{A}}(x,\tau_i)$ computed using $\tilde{\mathcal{A}}$, the remaining question is how we actually choose the parameter $\hat\tau(Z)$. We could construct $\tilde{f}$ as described in the proof of Theorem~\ref{thm:main}; then, by Corollary~\ref{cor:main}, $\pi(\tilde{\mathcal{A}}(x,\hat{\tau}(Z)-\gamma))$ is PAC for sufficiently small $\gamma$.\footnote{In fact, we can take $\gamma=\min_i|\tau_i-\tau_{i+1}|$, where $\tau_i$ are our choices in the design of $\tilde{\mathcal{A}}$ (not the proof of Theorem~\ref{thm:main}).}

However, constructing $\tilde{f}$ can be computationally intractable. To avoid this issue, note that we do not need to explicitly construct $\tilde{f}$; instead, it suffices for such a $\tilde{f}$ to exist (as long as we can find a way to compute the parameter $\hat\tau(Z)$).

To compute $\hat\tau(Z)$, it suffices to be able to compute the number of errors in the validation set for a candidate value of $\tau$. Then, we can solve the maximization problem over $\tau$ in (\ref{eqn:pacalgo}) by enumerating over the fixed choices of parameters $\tau_i$ used by $\tilde{\mathcal{A}}$; since the other values of $\tau$ produce the same prediction sets, we can ignore them.

Given a validation example $(x,y)\in Z$, computing $\mathbbm{1}(y\in F_{\tau_i}(x))$ is straightforward---since $F_{\tau_i}(x)=\pi(y_i)$, we simply check whether $y\in\pi(y_i)$. Doing so is a straightforward tree comparison operation---i.e., we enumerate nodes in $y_i$ (excluding holes) and check if they are all contained in $y$; if so, then $y\in\pi(y_i)$, and if not, then $y\not\in\pi(y_i)$.

In Algorithm~\ref{alg:main}, this search is implemented at the end: it returns the largest $\tau_i$ that achieves a given number of errors $k(\epsilon,\delta,n)$ on the validation dataset, where
\begin{align*}
k(\epsilon,\delta,n)=\max_{k\in\mathbb{N}\cup\{0\}}k
~\text{subj. to}~
\sum_{i=0}^k\binom{n}{i}\epsilon^i(1-\epsilon)^{n-i}<\delta
\end{align*}
is the number of mistakes permitted by an existing PAC prediction set algorithm~\cite{bastani-pac}.

\section{Experiments}

We evaluate our proposed approach on two tasks: semantic parsing for SQL, and Python code completion.


\textbf{SQL generation.}
In this task, the input to the model is a natural language utterance and the target is an SQL program. For the scoring function $f$, we use a state-of-the-art model PICARD~\cite{picard}, which modifies the decoding process of T5~\cite{https://doi.org/10.48550/arxiv.1910.10683} to constrain decoding to valid SQL programs. This model is trained on the Spider dataset~\cite{spider}, a large multi-domain and cross-database dataset for SQL semantic parsing. For our experiments, we use 7,000 examples from Spider to construct prediction sets.

\textbf{Python generation.}
In this task, the input to the model is a natural language problem statement combined with some $k$ lines of code, and the target is the remaining lines of code to complete the solution. We use the Codex~\cite{https://doi.org/10.48550/arxiv.2107.03374}, a GPT language model fine-tuned on publicly available GitHub code with proficiency in over a dozen programming languages including Python, JavaScript, Go, TypeScript, and C\#. For our experiments, we use natural language to Python code datasets including APPS~\cite{hendrycksapps2021} and HumanEval: Hand-Written Evaluation Set~\cite{https://doi.org/10.48550/arxiv.2107.03374}.

\textbf{Baseline.}
We compare to a baseline strategy that greedily chooses the least likely node to remove at each step. More precisely, it uses the same high-level strategy as our approach---i.e., construct a sequence of partial programs where each one contains the previous, and then use an existing prediction set algorithm to choose $\tau$. The difference is in how the sequence of partial programs is constructed---whereas our approach uses optimization to do so, the baseline uses a greedy strategy. In particular, among all nodes of the current partial program with at least one child missing, it chooses to remove the one that most reduces the negative log likelihood (NLL) of the partial program (normalized by the number of nodes removed).

\begin{figure*}[h]
\begin{subfigure}{.45\textwidth}
\centering
\includegraphics[width=3in]{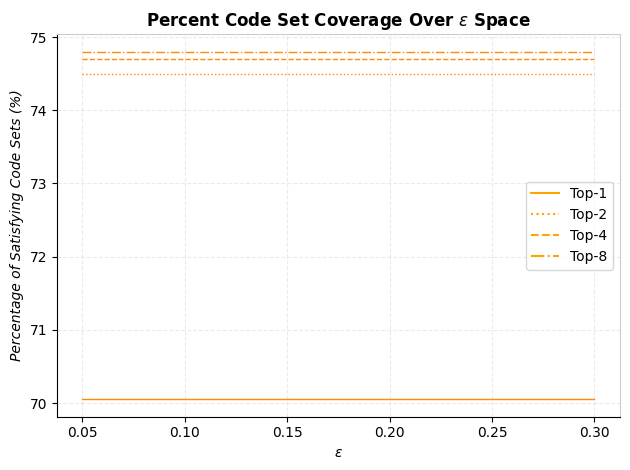}
\caption{Results for Picard on the Spider dataset.}
\label{fig:picard-baseline}
\end{subfigure}
\hfill
\begin{subfigure}{.45\textwidth}
\centering
\includegraphics[width=3in]{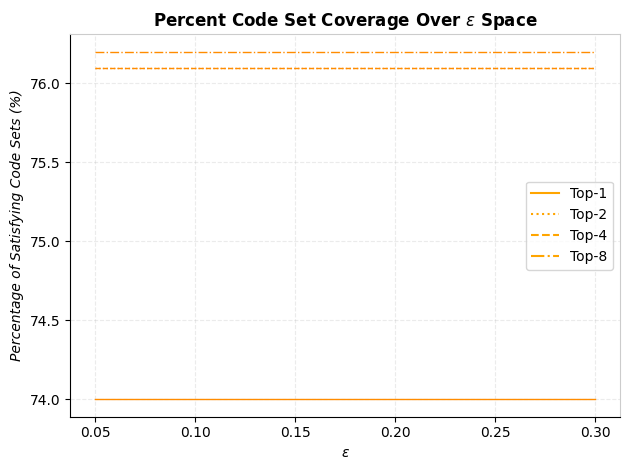}
\caption{Results for Codex on the APPS dataset.}
\label{fig:python-baseline}
\end{subfigure}
\caption{Coverage as a function of $\epsilon$ for the top-$k$ predictions by the base model. By construction, the coverage does not depend on $\epsilon$.}
\label{fig:singleton-baselines}
\end{figure*}

\textbf{Hyperparameters.}
The main hyperparameter is the choice of search space over $\tau$; we used the simple and natural choice of covering uniformly in increments of 0.01. In general, we found that the results are not overly sensitive to the choice of thresholds, as long as they largely covered the possible choices (with the tradeoff that too few choices can lead to suboptimality, whereas too many can increase conservativeness since the search space is larger). In addition, we choose $\delta=0.01$; we vary $\epsilon$ in our results.

\textbf{Results.}
We implemented the structured prediction set algorithm described in section~\ref{algo} to construct PAC prediction sets of code leveraging the abstract syntax tree of the SQL and Python programming languages. The results in Figure~\ref{fig:spider_node_rm} \&~\ref{fig:codex_node_rm} show the fraction of nodes removed from the predicted AST for varying $\epsilon$ and $m$. The results in Figures~\ref{fig:spider_set_coverage} \&~\ref{fig:codex_set_coverage} show the percentage of constructed code sets that include the target program for varying $\epsilon$ and $m$.

For both datasets, our approach constructs prediction sets that remove significantly fewer nodes than the baseline. For instance, with $\epsilon=0.15$, our approach only removes about 45\% of nodes for SQL (vs. 55\% for the baseline), and about 55\% of nodes for Python (vs. 65\% for the baseline). Furthermore, our approach achieves better performance while achieving similar coverage rate (both our approach and the baseline achieve the desired coverage rate in all cases). There is remaining slack compared to the desired coverage rate to account for generalization error; this slack can be reduced by using more samples.

Additionally, we note that the coverage rate decreases as $\epsilon$ increases, demonstrating that our approach is not overly conservative. In particular, the number of nodes removed is due in large part to errors in the underlying language model.

Interestingly, the performance of our algorithm is not very sensitive to $m$, suggesting that $m=1$ typically suffices for constructing prediction sets. Intuitively, this finding suggests that most errors are localized, though more analysis is needed to understand this phenomenon.

Finally, we show the coverage of the top-$k$ predictions of the baseline in Figure~\ref{fig:singleton-baselines}. As can be seen, the coverage of just the top-$1$ example is around 70\% for PICARD and around 74\% for Codex. The coverage can be improved slightly, but not substantially, by increasing $k$. This result justifies our choice of using partial programs to represent prediction sets.

\section{Conclusion}

In this work, we presented an algorithm for constructing PAC prediction sets for large language models for code generation and semantic parsing. Our approach constructs compact prediction sets in the form of partial programs, leveraging an optimization-based approach to construct a monotonic set of candidate prediction sets and then using an existing algorithm to rigorously select a prediction set threshold. Our experiments demonstrate that our approach constructs valid prediction sets while significantly outperforming a na\"{i}ve baseline that uses a greedy heuristic instead of optimization. While our approach is tailored to code generation, we anticipate that the general principle of using optimization to construct compact prediction sets can be applied to many structure prediction problems.

\textbf{Limitations.} We briefly summarize several limitations of our approach. First, users are required to specify an error tolerance for our algorithm to be applied.
Furthermore, our approach relies heavily on the performance of the underlying, well-trained generative model, and may exhibit limitations if the base model is not sufficiently accurate. Finally, we have studied only one way of representing prediction sets, and others may be possible. User studies would be needed to compare the effectiveness of different choices.

\section*{Acknowledgements}
We are grateful to Jeevana Inala for her helpful comments and discussions. This work was supported in part by NSF Award CCF-1910769, NSF Award CCF-1917852, and ARO Award W911NF-20-1-0080.

\bibliography{refs}
\bibliographystyle{icml2023}


\newpage
\appendix
\onecolumn
\section{Sample Outputs}

\begin{figure*}[h]
\begin{subfigure}{.45\textwidth}
\centering
\includegraphics[width=3.1 in]{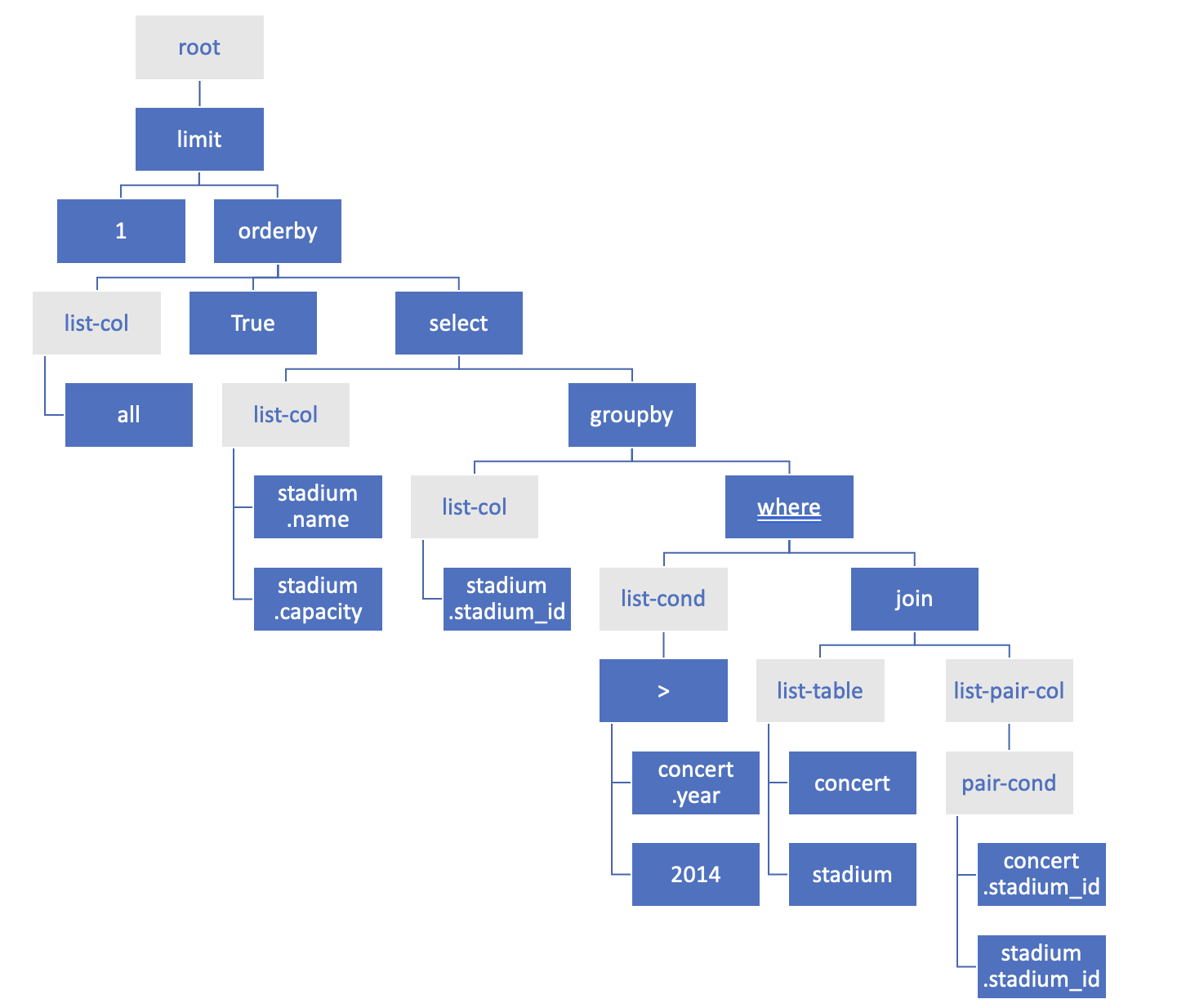}
{\scriptsize
\begin{align*}
&\texttt{SELECT t1.name, t1.capacity FROM stadium AS t1} \\
&\texttt{JOIN concert AS t2 ON t1.stadium\_id = t2.stadium\_id} \\
&\texttt{WHERE t2.year > 2014} \\
&\texttt{GROUP BY t1.stadium\_id} \\
&\texttt{ORDER BY COUNT(*)} \\
&\texttt{DESC LIMIT 1;}
\end{align*}}
\caption{Ground truth abstract syntax tree (AST) and SQL query from the Spider dataset \cite{spider}}
\label{fig:target-ast-spider-2}
\end{subfigure}
\hfill
\begin{subfigure}{.45\textwidth}
\centering
\includegraphics[width=3.1 in]{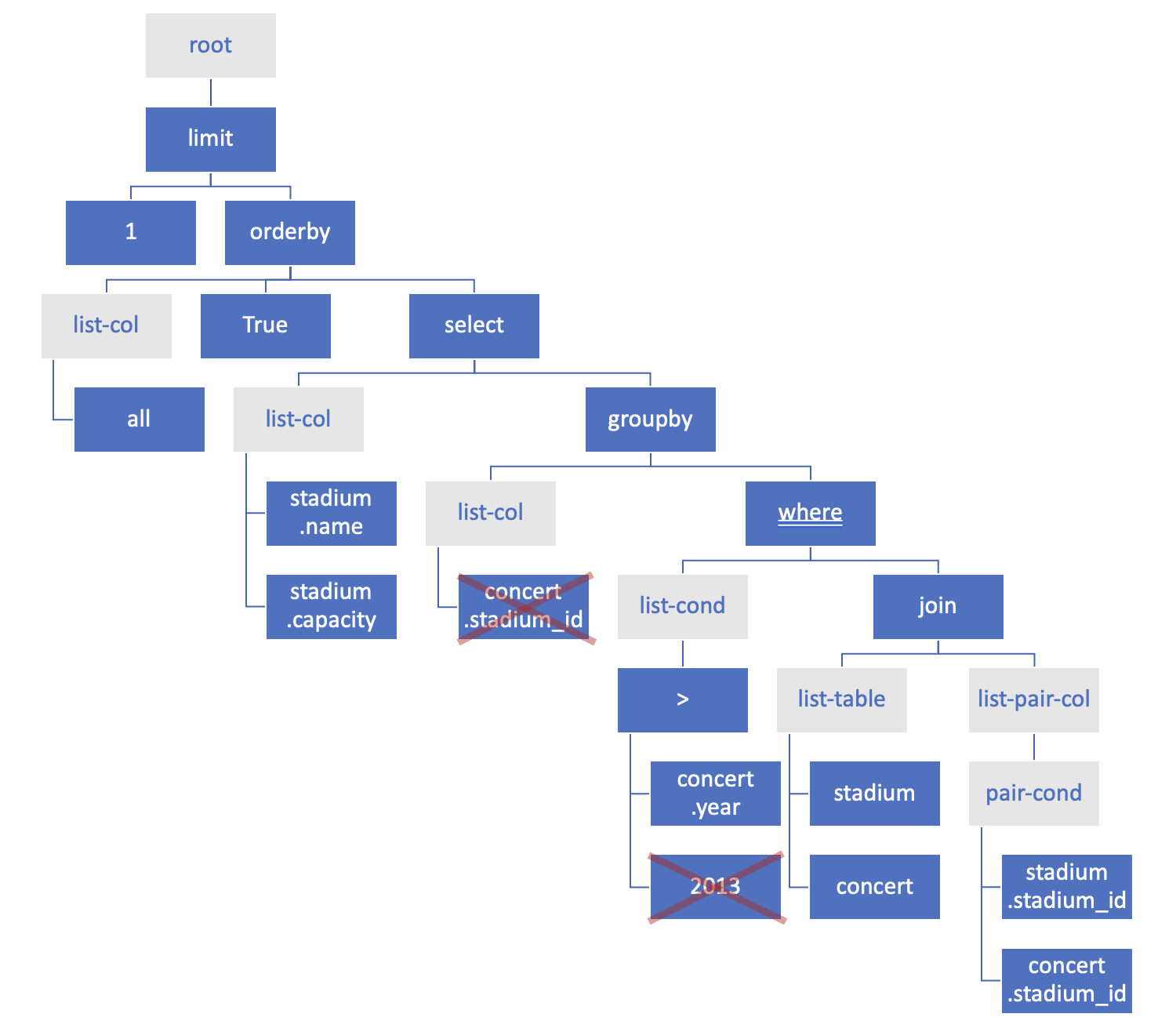}
{\scriptsize
\begin{align*}
&\texttt{SELECT t1.name, t1.capacity FROM stadium AS t1} \\
&\texttt{JOIN concert AS t2 ON t1.stadium\_id = t2.stadium\_id} \\
&\texttt{WHERE t2.year > 2013} \\
&\texttt{GROUP BY t2.stadium\_id} \\
&\texttt{ORDER BY COUNT(*)} \\
&\texttt{DESC LIMIT 1;} \\\\
&\texttt{SELECT t1.name, t1.capacity FROM stadium AS t1} \\
&\texttt{JOIN concert AS t2 ON t1.stadium\_id = t2.stadium\_id} \\
&\texttt{WHERE t2.year > ??} \\
&\texttt{GROUP BY ??} \\
&\texttt{ORDER BY COUNT(*)} \\
&\texttt{DESC LIMIT 1;}
\end{align*}}
\caption{Predicted AST, predicted SQL Query, and resulting prediction set for the same task as in \ref{fig:target-ast-spider-2} with $m=2$ holes. Note that the model predicts the full AST depicted above, and our PAC structured prediction set algorithm removes two subtrees (nodes removed shown with an ``x").}
\label{fig:predicted-set-spider-2}
\end{subfigure}
\caption{Example of ground truth SQL query, predicted SQL query, and the constructed prediction set. Note that without the subtree removal in the predicted AST, the resulting query is incorrect. With the two holes in the SQL query, the code prediction set contains the ground truth query.}
\label{fig:spider-example-2}
\end{figure*}

\begin{figure*}[h]
\begin{subfigure}{.45\textwidth}
\centering
\includegraphics[width=2 in]{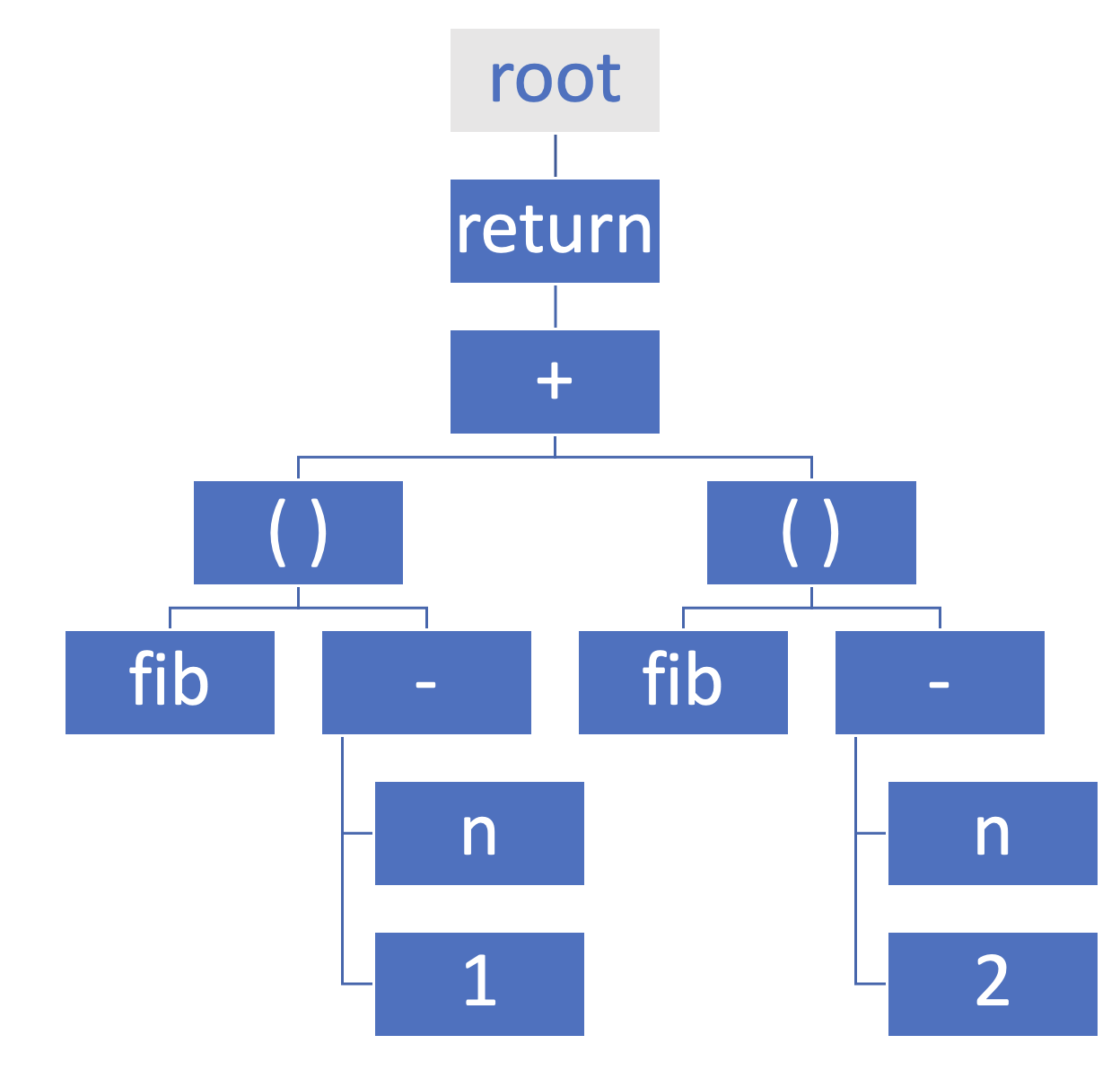}
{\scriptsize
\begin{align*}
&\texttt{return fib(n-1) + fib(n-2)} \\
\end{align*}}
\caption{Ground truth abstract syntax tree (AST) and Python line of code in the APPS dataset \cite{https://doi.org/10.48550/arxiv.2107.03374}.}
\label{fig:target-ast-codex-m1}
\end{subfigure}
\hfill
\begin{subfigure}{.45\textwidth}
\centering
\includegraphics[width=2 in]{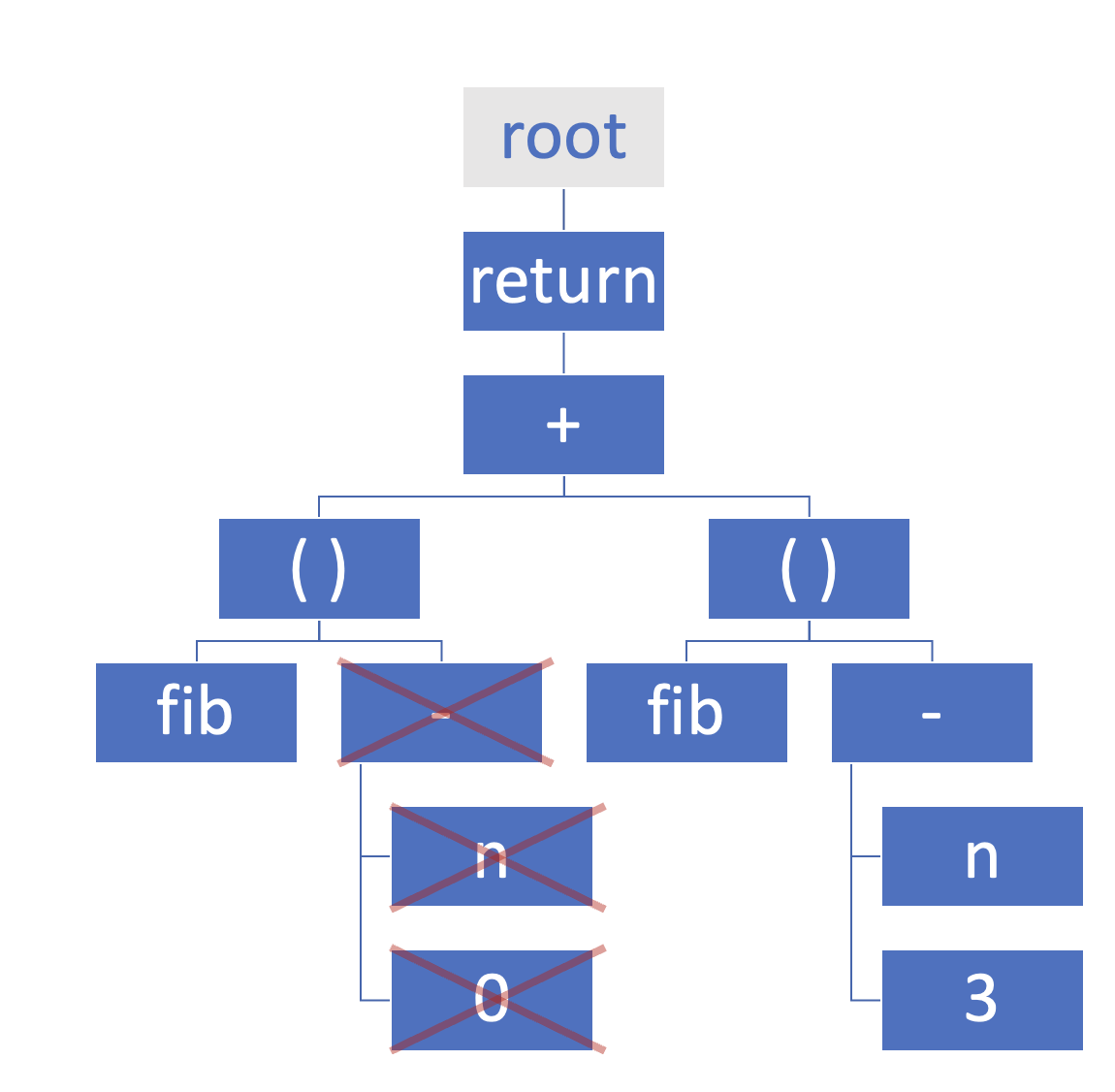}
{\scriptsize
\begin{align*}
&\texttt{return fib(n-0) + fib(n-3)} \\\\
&\texttt{return fib(??) + fib(n-3)}
\end{align*}}
\caption{Predicted AST, predicted Python code, and prediction set for the same task as in \ref{fig:target-ast-codex-m1} with $m=1$ hole1. Note that the model predicts the full AST depicted above, and our PAC structured prediction set algorithm removes a subtree (nodes removed shown with an ``x").}
\label{fig:predicted-set-codex-m1}
\end{subfigure}
\caption{Example of ground truth Python line of code, predicted line of code, and constructed prediction set with a single subtree removal, $m=1$. Note that with a single subtree removal, the resulting code prediction set does not contain the ground truth code.}
\label{fig:codex-example-m1}
\end{figure*}

\begin{figure*}[h]
\begin{subfigure}{.45\textwidth}
\centering
\includegraphics[width=2 in]{exhibits/codex_ast_target.png}
{\scriptsize
\begin{align*}
&\texttt{return fib(n-1) + fib(n-2)} \\
\end{align*}}
\caption{Ground truth abstract syntax tree (AST) and Python line of code in the APPS dataset \cite{https://doi.org/10.48550/arxiv.2107.03374}.}
\label{fig:target-ast-codex}
\end{subfigure}
\hfill
\begin{subfigure}{.45\textwidth}
\centering
\includegraphics[width=2 in]{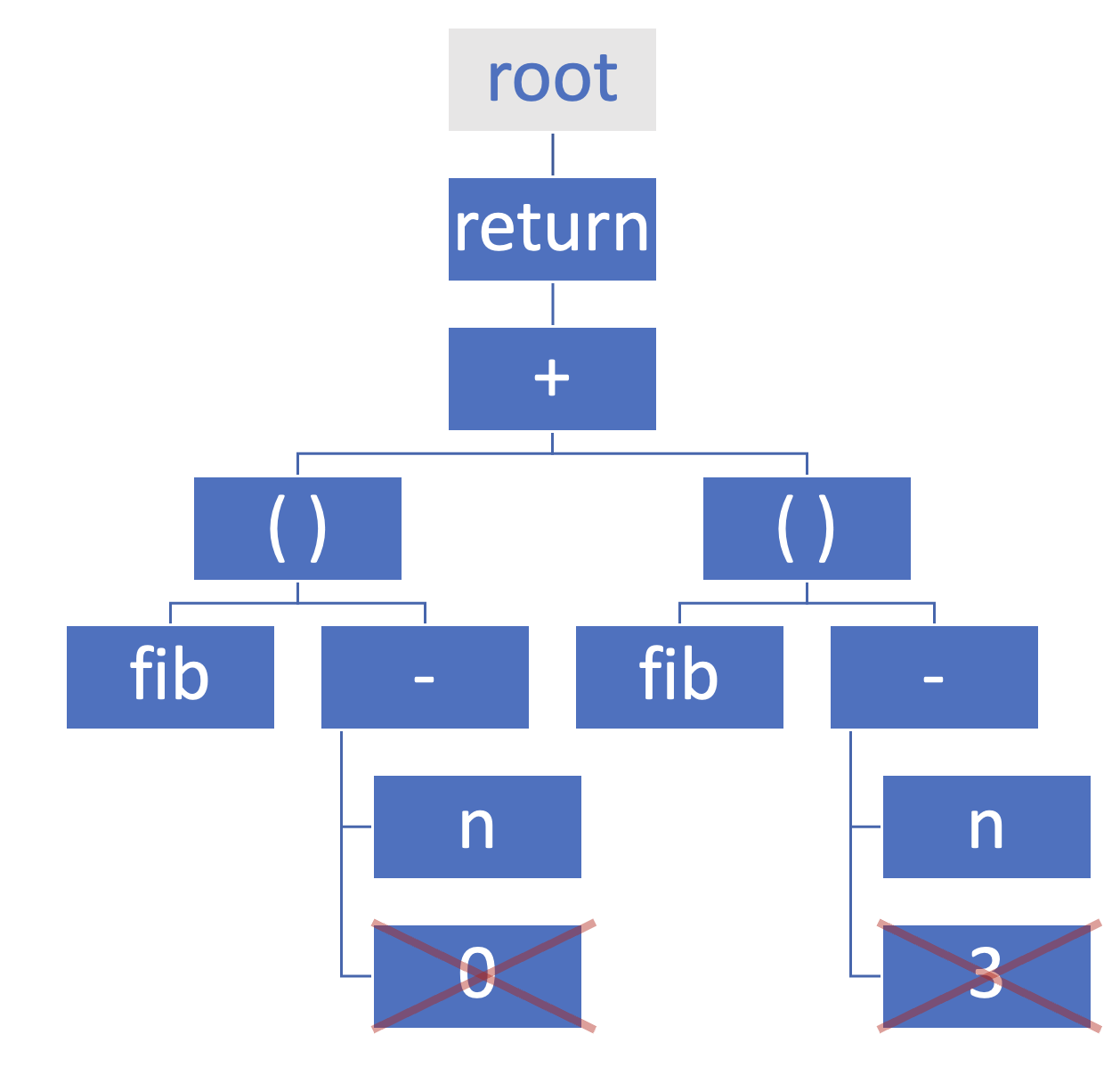}
{\scriptsize
\begin{align*}
&\texttt{return fib(n-0) + fib(n-3)} \\\\
&\texttt{return fib(n-??) + fib(n-??)}
\end{align*}}
\caption{Predicted AST, predicted Python code, and prediction set for the same task as in \ref{fig:target-ast-codex} with $m=2$ holes. Note that the model predicts the full AST depicted above, and our PAC structured prediction set algorithm removes two subtrees (nodes removed shown with an ``x").}
\label{fig:predicted-set-codex}
\end{subfigure}
\caption{Example of ground truth Python line of code, predicted line of code, and constructed prediction set with two subtrees removed, $m=2$. Note that without the subtree removal in the predicted AST, the resulting query is incorrect. With the two holes in the Python line of code, the code prediction set contains the ground truth code.}
\label{fig:codex-example-m2}
\end{figure*}

\begin{figure*}[h]
\begin{subfigure}{.45\textwidth}
\centering
\includegraphics[width=2.5 in]{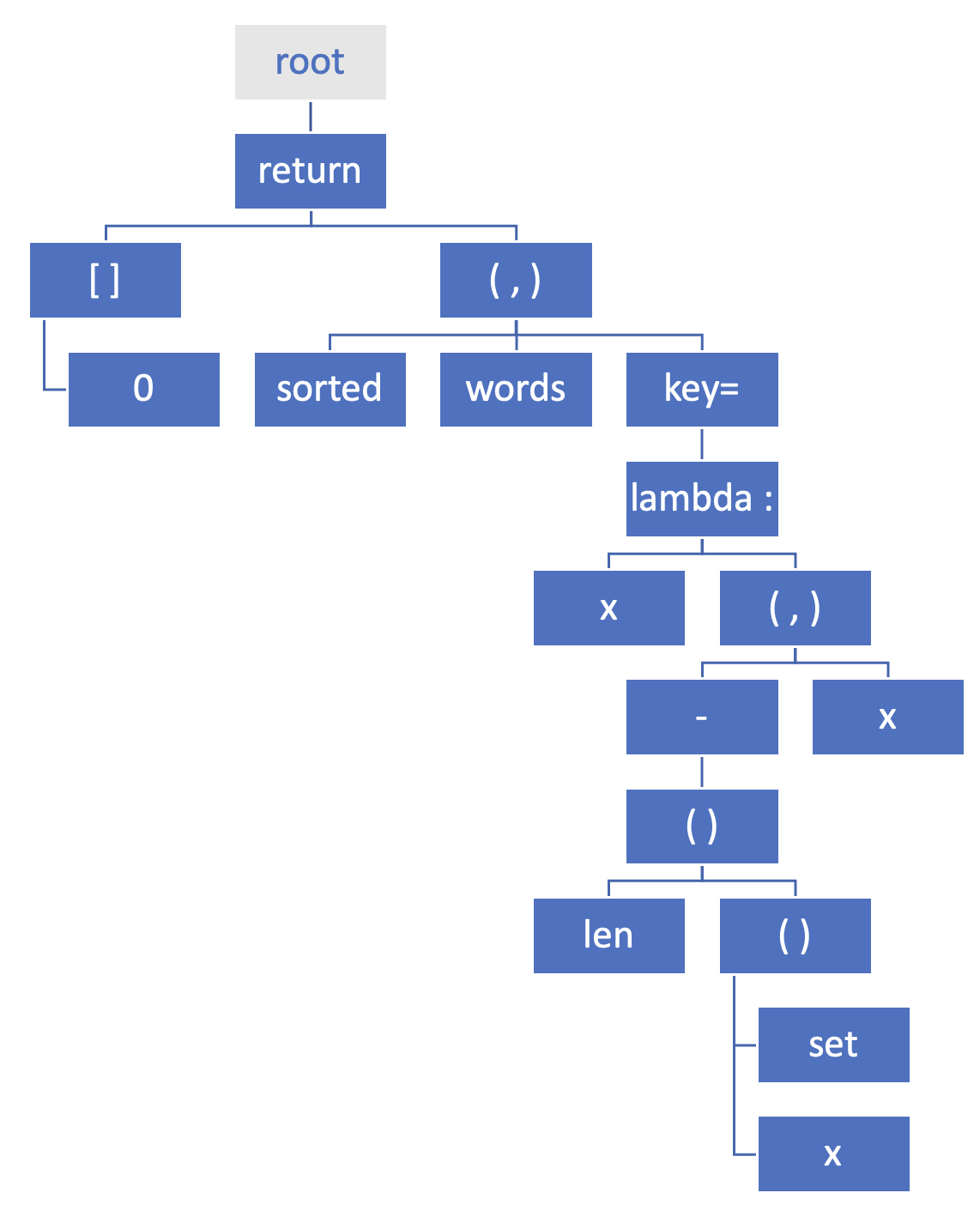}
{\scriptsize
\begin{align*}
&\texttt{return sorted(words, key = lambda x: (-len(set(x)), x))[0]}
\end{align*}}
\caption{Ground truth abstract syntax tree (AST) and Python line of code from the APPS dataset \cite{https://doi.org/10.48550/arxiv.2107.03374}.}
\label{fig:target-ast-codex-2}
\end{subfigure}
\hfill
\begin{subfigure}{.45\textwidth}
\centering
\includegraphics[width=2.5 in]{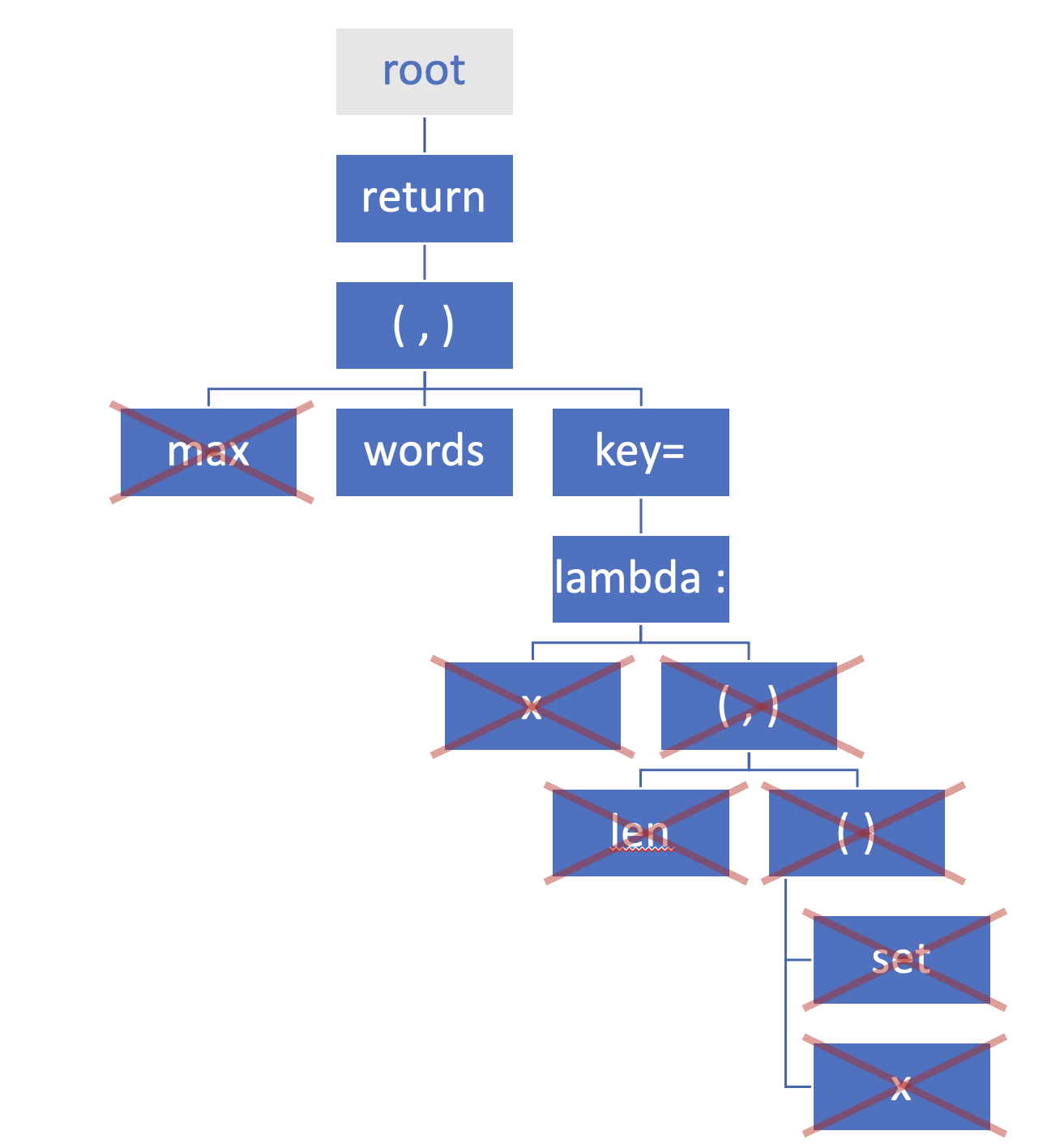}
{\scriptsize
\begin{align*}
&\texttt{return max(words, key=lambda x: len(set(x)))} \\
&\texttt{return ??(words, key=lambda : ??)}
\end{align*}}
\caption{Predicted AST and resulting code-prediction set for the same task as in \ref{fig:target-ast-codex-2} with $m=3$ holes. Note that the model predicts the full AST depicted above, and our PAC structured prediction set algorithm removes three subtrees (nodes removed shown with an ``x").}
\label{fig:predicted-set-codex-2}
\end{subfigure}
\caption{Example of ground truth Python line of code, predicted line of code, and constructed prediction set with three subtrees removed, $m=3$. Note that without the subtree removals in the predicted AST, the resulting query is incorrect. With the three holes in the Python line of code, the code prediction set contains the ground truth code.}
\label{fig:codex-example-2}
\end{figure*}


\end{document}